\documentclass[runningheads]{llncs}

\usepackage[utf8]{inputenc}
\usepackage{amsmath,amssymb}
\usepackage{ifthen}
\newboolean{anonymous}
\usepackage{etoolbox}
\makeatletter
\patchcmd{\@algocf@start}
  {-1.5em}
  {0pt}
  {}{}
\makeatother
\usepackage{mathtools}
\usepackage{booktabs}
\usepackage{csquotes}

\usepackage{comment} 
\usepackage{caption}
\usepackage{url}

\usepackage{hyperref}
\usepackage{cleveref}
\usepackage{siunitx}

\usepackage{algorithm}
\usepackage[noend,compatible]{algpseudocode}
\algnewcommand{\IIF}[1]{\State\algorithmicif\ #1\ \algorithmicthen}
\algnewcommand{\ENDIIF}{\unskip\ \algorithmicend\ \algorithmicif}
\Crefname{ALC@unique}{Line}{Lines}
\usepackage{eqparbox}
\usepackage{graphicx,microtype,xspace}
\usepackage[colorinlistoftodos]{todonotes}

\usepackage{tikz-cd}
\usepackage{enumitem}
\usepackage{multicol}

\usepackage{pifont}

\newcommand{\IPFP}{\texttt{IPFP}\xspace}

\newcommand{\BRANCHUNIFORM}{\texttt{BRANCH-UNIFORM}\xspace}
\newcommand{\aids}{\textsc{aids}\xspace}
\newcommand{\acyclic}{\textsc{acycl}\xspace}
\newcommand{\mao}{\textsc{mao}\xspace}
\newcommand{\muta}{\textsc{muta}\xspace}
\newcommand{\pah}{\textsc{pah}\xspace}
\newcommand{\stocks}{\textsc{stocks}\xspace}
\newcommand{\stocksint}{\textsc{stocks-i}\xspace}
\newcommand{\stocksfloat}{\textsc{stocks-f}\xspace}
\newcommand{\stocksunl}{\textsc{stocks-n}\xspace}

\usepackage{tikz,pgf,pgfplots,pgfplotstable}
\usetikzlibrary{calc,backgrounds,arrows,decorations.markings,decorations.pathreplacing,matrix,shapes,graphs,positioning,fit,snakes}
\usepgfplotslibrary{groupplots}
\usepgfplotslibrary{statistics}
\usepgfplotslibrary{colorbrewer}
\pgfplotsset{compat=1.15, cycle list/Set1-8}
\definecolor{DarkGray}{RGB}{10,60,50}

\definecolor{C1}{RGB}{27,158,119}
\definecolor{C2}{RGB}{217,95,2}
\definecolor{C3}{RGB}{117,112,179}
\definecolor{C4}{RGB}{231,41,138}
\definecolor{C5}{RGB}{102,166,30}
\definecolor{C6}{RGB}{230,171,2}
\definecolor{C7}{RGB}{166,118,29}
\definecolor{C8}{RGB}{102,102,102}
\tikzstyle{ABCTARBZ}=[mark=triangle, mark size=2.5pt, semithick, color=C1]
\tikzstyle{ABCONLY}=[mark=x, mark size=2.5pt, semithick, color=C2]
\tikzstyle{TARBZ}=[mark=diamond, mark size=2.5pt, semithick, color=C3]

\tikzstyle{ABCNOREF}=[mark=star, mark size=2.5pt, semithick, color=C6]
\tikzstyle{ABCREF}=[mark=asterisk, mark size=2.5pt, semithick, color=C4]
\pgfplotsset{every tick label/.append style={font=\scriptsize}}

\newcommand{\ratioplots}[2]{%
\nextgroupplot[
title={#1},
title style={yshift=.05cm},
ylabel={compression ratio},
xlabel style={align=center,font=\scriptsize,yshift=.1cm},
ylabel style={font=\scriptsize,yshift=-.1cm},
xlabel={$k/|X|$ in \%},
ymin={0},
ymax={0.032},
legend style={align=left, draw=none, column sep=1em,font=\small,legend columns=3, legend cell align=left},
legend to name={ratiolegend},
xtick={20,40,60,80,100}
]
\addlegendimage{ABCTARBZ};
\addlegendentry{\ABC w/ tar.bz};
\addlegendimage{ABCONLY};
\addlegendentry{\ABC w/o tar.bz};
\addlegendimage{TARBZ};
\addlegendentry{tar.bz};
\addplot[ABCTARBZ] plot [error bars/.cd, y dir = both, y explicit] table [x=percent_graphs, y=mean_file_size_ratio_<abc_refined_plus_tar>, y error=std_file_size_ratio_<abc_refined_plus_tar>, col sep=comma] {data/#2};
\addplot[ABCONLY] table [x=percent_graphs, y=ABC_only_compression_rate, col sep=comma] {data/#2};
\addplot[TARBZ] table [x=percent_graphs, y=compression_ratio_<tar.bz>, col sep=comma] {data/#2};
}

\newcommand{\ctimeplots}[2]{%
\nextgroupplot[
title={#1},
title style={yshift=-.1cm},
ylabel={runtime},
xlabel style={align=center,font=\scriptsize,yshift=.1cm},
ylabel style={font=\scriptsize,yshift=-.1cm},
xlabel={$k/|X|$ in \%},
legend style={align=left, draw=none, column sep=1em,font=\small,legend columns=2, legend cell align=left},
legend to name={ctimelegend},
xtick={20,40,60,80,100}
]
\addlegendimage{ABCNOREF};
\addlegendentry{\ABC{} -- arborescence};
\addlegendimage{ABCREF};
\addlegendentry{\ABC{} -- refinement};
\addlegendimage{ABCONLY};
\addlegendentry{\ABC{} -- total};
\addlegendimage{TARBZ};
\addlegendentry{tar.bz};
\addplot[ABCNOREF] plot [error bars/.cd, y dir = both, y explicit] table [x=percent_graphs, y=mean_gedlib_runtime, y error=std_gedlib_runtime, col sep=comma] {data/#2};
\addplot[ABCREF] plot [error bars/.cd, y dir = both, y explicit] table [x=percent_graphs, y=mean_gedlib_runtime_refinement, y error=std_gedlib_runtime_refinement, col sep=comma] {data/#2};
\addplot[ABCONLY] plot [error bars/.cd, y dir = both, y explicit] table [x=percent_graphs, y=mean_compression_time_<abc_refined_plus_tar>, y error=std_compression_time_<abc_refined_plus_tar>, col sep=comma] {data/#2};
\addplot[TARBZ] table [x=percent_graphs, y=compression_time_<tar.bz>, col sep=comma] {data/#2};
}

\usepackage{cancel}
\DeclareMathOperator*{\argmin}{arg\,min}

\newcommand{\bigA}{\ensuremath{\mathcal{A}}\xspace}

\newcommand{\bigS}{\ensuremath{\mathcal{S}}\xspace}
\newcommand{\NP}{\ensuremath{\mathcal{NP}}\xspace}

\newcommand{\spaceA}{\ensuremath{\mathbb{A}}\xspace}

\newcommand{\spaceG}{\ensuremath{\mathbb{G}}\xspace}

\newcommand{\abclong}{\textsc{Arborescence-Based Compression Problem}\xspace}
\newcommand{\mealong}{\textsc{Minimum Edit Arborescence Problem}\xspace}
\newcommand{\msalong}{\textsc{Minimum Spanning Arborescence Problem}\xspace}
\newcommand{\mgealong}{\textsc{Minimum Graph Edit Arborescence Problem}\xspace}
\newcommand{\ABC}{\ensuremath{\mathit{ABC}}\xspace}
\newcommand{\MEA}{\ensuremath{\mathit{MEA}}\xspace}
\newcommand{\MSA}{\ensuremath{\mathit{MSA}}\xspace}
\newcommand{\MGEA}{\ensuremath{\mathit{MGEA}}\xspace}

\newcommand{\ie}{i.\,e.\@\xspace}

\newcommand{\eg}{e.\,g.\@\xspace}
\newcommand{\wrt}{w.\,r.\,t.\@\xspace}
\newcommand{\mywlog}{w.\,l.\,o.\,g.\@\xspace}

\newcommand{\GED}{\ifmmode\mathrm{GED}\else{GED}\fi\xspace}

\newcommand{\ALGONE}{\ensuremath{\operatorname{\mathtt{ALG-1}}}\xspace}
\newcommand{\ALGDIST}{\ensuremath{\operatorname{\mathtt{ALG-DIST}}}\xspace}
\newcommand{\ALGTWO}{\ensuremath{\operatorname{\mathtt{ALG-2}}}\xspace}

\newcommand{\enc}{\textnormal{C}}

\begin{document}

\title{The Minimum Edit Arborescence Problem and Its Use in Compressing Graph Collections [Extended Version]\thanks{Supported  by  Agence  Nationale  de  la  Recherche  (ANR),  projects STAP ANR-17-CE23-0021 and DELCO ANR-19-CE23-0016. FY acknowledges  the  support  of  the  ANR  as  part  of  the \enquote{Investissements  d’avenir} program, reference ANR-19-P3IA-0001 (PRAIRIE 3IA Institute).}}
\titlerunning{The Minimum Edit Arborescence Problem}

\author{%
Lucas Gnecco\inst{1} \orcidID{0000-0002-1561-2080} \and 
Nicolas Boria\inst{1} \orcidID{0000-0002-0548-4257} \and 
Sébastien Bougleux\inst{2} \orcidID{0000-0002-4581-7570} \and 
Florian Yger\inst{1} \orcidID{0000-0002-7182-8062} \and 
David B. Blumenthal\inst{3} \orcidID{0000-0001-8651-750X}}

\authorrunning{L. Gnecco et al.}

\institute{%
PSL Université Paris-Dauphine, LAMSADE, Paris, France \\ \email{\{lucas.gnecco,nicolas.boria,florian.yger\}@dauphine.fr} \and
UNICAEN, ENSICAEN, CNRS, GREYC, Caen, France \\ \email{sebastien.bougleux@unicaen.fr} \and
Department Artificial Intelligence in Biomedical Engineering (AIBE), Friedrich-Alexander University Erlangen-Nürnberg (FAU), Erlangen, Germany \\ \email{david.b.blumenthal@fau.de}}

\maketitle

\begin{abstract} The inference of minimum spanning arborescences within a set of objects is a general problem which translates into numerous application-specific unsupervised learning tasks. We introduce a unified and generic structure called \emph{edit arborescence} that relies on edit paths between data in a collection, as well as the \mealong, which asks for an edit arborescence that minimizes the sum of costs of its inner edit paths. Through the use of suitable cost functions, this generic framework allows to model a variety of problems. In particular, we show that by introducing \emph{encoding size preserving edit costs}, it can be used as an efficient method for compressing collections of labeled graphs. Experiments on various graph datasets, with comparisons to standard compression tools, show the potential of our method.
\end{abstract}

\section{Introduction}\label{sec:intro}

The discovery of some underlying structure within a collection of data is the main goal of unsupervised learning. 
Among the different kinds of graph structures available for structure inference, arborescences play an essential role, because they contain the minimal number of edges required to connect all the entries of the collection and induce a meaningful hierarchy within the data. For these reasons, arborescences are widely used in structure inference, in numerous fields ranging from bioinformatics \cite{CORNWELL2017R333} to computational linguistics \cite{moschitti-etal-2006-semantic}.
For constructing arborescences, distances have to be computed for data objects within the collection. While the computation of distances is trivial for many kinds of simple data (\eg, vectors in Euclidean space), it is often challenging for more complex kinds of data such as strings, trees, or graphs. For such data, \emph{edit distances}\,---\,measuring the distance between two objects $o_1$ and $o_2$ as the cost of modifications needed to transform $o_1$ into $o_2$\,---\,provide meaningful measures.

In this work, we propose a unified and generic framework for minimum arborescence computation on collections of structured data for which an edit distance is available. We introduce the concept of \emph{edit arborescence} which generalizes the concept of edit path (a sequence of edit operations, or modifications), and we formalize the $\mealong$ (\MEA).
By using appropriate edit cost functions over the edit operations, as well as different sets of allowed edit operations, this generic framework allows to tackle a variety of specific problems, such as event detection in time series \cite{guralnik:1999wp}, morphological forests inference over a language vocabulary \cite{luo:2017uk}, or structured data compression \cite{chwatala09}.

As a proof of concept, we focus on the latter application, and address the problem of compressing a collection of labeled graphs. 
To the best of our knowledge, this problem has not been addressed in the literature. In graph stores, each graph is encoded individually using space-efficient representations based on different, mainly lossless compression schemes \cite{besta18,besta19b}, 
but without taking into account the other graphs in the store. This is also the case for lossy graph compression schemes 
\cite{sourek2021lossless}. All of these compression schemes are beyond the main focus of this paper, and we refer the reader to the above references.

Contrary to these schemes, our compression method relies heavily on reference-based compression underpinned by an arborescence connecting the graphs of the collection. Intuitively, each graph is represented by an edit path between its parent graph and itself. Each graph can thus be reconstructed recursively up to the root element of the arborescence, which we define as the empty graph. Similar ideas have been proposed for compressing web graphs seen as a temporal graphs with edge insertions and deletions \cite{adler01}, or collections of bitvectors using the Hamming distance \cite{BOOKSTEIN91}, recently applied to graph annotations (colors) \cite{Almodaresi20}. While these approaches can be considered as early examples of using \MEA in compression, our formulation is more general.

We first formalize the concepts of edit distances and arborescences in \Cref{sec:prelim}. In \Cref{sec:mea}, we introduce edit arborescences and define \MEA. \Cref{sec:mgea} deals specifically with graph data and the graph edit distance, and formalizes the \mgealong (\MGEA). \Cref{sec:abc} provides detailed explanations on how to make use of the \MGEA to address the compression of a set of labeled graphs. In \Cref{sec:xp}, we report the results of the experimental evaluation. Finally, \Cref{sec:conc} concludes the paper and points out to future work. 

 

\section{Preliminaries}\label{sec:prelim}


We consider data (sequence, tree, graph) defined by a combinatorial structure and labels attached to the elements of this structure. Labels may be of any type. Unlabeled and unstructured data are special cases. 


\paragraph{\textbf{Edit Distance.}}
Given a space $\Omega$ of all data of a fixed type, an \emph{edit path} is a sequence of elementary modifications (or \emph{edit operations}) transforming an object of $\Omega$ into another one. Typical edit operations are the deletion and the insertion of an element of the structure, and the substitution of an attribute attached to an element. Given a cost function $c\geq 0$ defined on edit operations, the \emph{edit distance} $d_c:\Omega\times\Omega\rightarrow\mathbb{R}_{\geq 0}$ measures the minimal total cost required to transform $x\in\Omega$ into $y\in\Omega$, up to an equivalence relation:
    $d_c(x,y)\coloneqq {\min}_{P\in\mathcal{P}(x,y)}\,c(P)$, 
with $c(P)\coloneqq\sum_{o\in P}c(o)$ the cost of an edit path $P$, and $\mathcal{P}(x,y)$ the set of all edit paths transforming $x$ into an element of $[y]\coloneqq\{z\in\Omega\,|\,y\sim z\}$, the equivalence class of $y$ for an equivalence relation $\sim$ on $\Omega$.
Equality is the equivalence relation usually considered for strings or sequences (Hamming, Levenshtein or 
discrete time warping distances). 
Isomorphism is used for trees and graphs.

The set $\Omega$ equipped with an edit distance $d_c$ defines an \emph{edit space} $(\Omega,d_c)$. 
We assume that $d_c$ is metric or pseudometric (
if $\sim$ is not equality). 
The set $\Omega$ contains a null (or empty) element denoted by $0_\Omega\in\Omega$. Any other element of $(\Omega,d_c)$ can be constructed by insertion operations only from $0_\Omega$.

\paragraph{\textbf{Arborescences.}}
A directed graph (digraph) is a pair $G\coloneqq(V,E)$, where $V\coloneqq\{v_0,..,v_n\}$ is a set of nodes and $E\subseteq V\times V$ is a set of directed edges. Within such a graph, a \emph{spanning arborescence} is a rooted, oriented, spanning tree, \ie, a set of edges that induces exactly one directed path from a root node $r\in V$ to each other node in $V\setminus\{r\}$. By assuming w.\,l.\,o.\,g.\ that the root element is $v_0$ and reminding that all other nodes in an arborescence have a unique parent node, an arborescence can 
be represented by a sequence of node indices $\bigA$ such that, for all $i \in [1,n]$, $\bigA[i]$ denotes the index of the unique parent node of node $v_i$. The set of edges of $\bigA$ is denoted by $E^\bigA$.

\section{The \mealong}\label{sec:mea}
In this section, we introduce and describe the generic \mealong ~ (\MEA), a versatile problem. An instance of \MEA is a finite dataset $X$ living in an edit space $(\Omega,d_c)$. \MEA asks for a minimum-cost edit arborescence rooted at the null element.

Given a set $X\coloneqq\{x_0,x_1,...,x_n\}\subset\Omega$ such that $x_0\coloneqq 0_\Omega$, we define an \emph{edit arborescence} as a pair $(\bigA,\Psi)$, where 
    $\bigA$ is a sequence of $n$ indices that defines an arborescence rooted at the index $0$, such that for all $i\in[1,n]$, $\bigA[i]$ is the parent-index of $i$. 
    $\Psi\coloneqq(P_1,...,P_n)$ is a sequence of edit paths, such that $P_i\in \mathcal{P}(x_{\bigA[i]},x_i)$ holds for all $\in[1,n]$, 
    \ie, $P_i$ is an edit path between $x_i$ and its parent in $\bigA$.
$\spaceA(X)$ is the set of all edit arborescences on $X$.



\begin{definition}[\MEA]\label{def:mea}
Given a finite set $X\subset\Omega$ and an edit cost function $c$, the \mealong (\MEA) asks for an edit arborescence $(A^\star,\Psi^*)$ on $X\cup\{0_\Omega\}$, which is rooted at the null element $0_\Omega\in\Omega$ and has a minimum cost $c(\Psi^\star)$ among all $(\bigA,\Psi)\in\spaceA(X)$, with $c(\Psi)\coloneqq\sum_{P\in\Psi}c(P)$.
\end{definition}
By definition, it holds that $c(\Psi^\star)=\min_{(\bigA,\Psi)\in\spaceA_c(X)}\sum_{P_i\in\Psi}d_c(x_{\bigA[i]},x_i)$, where $\spaceA_c(X)$ is the set of edit arborescences in $(\Omega,d_c)$, \ie, edit arborescences with edit paths restricted to minimal-cost edit paths \wrt $c$. This generic definition can translate into various optimization problems, with different characteristics in terms of complexity and/or approximability, depending on the edit space. 

\paragraph{\textbf{Exact Solver.}}
Whenever exact edit distances and corresponding edit paths can be computed, the following procedure produces an optimal solution for $\MEA$:
\begin{enumerate}\itemsep0em
    \item Construct the complete directed weighted graph on the set $X\cup\{0_\Omega\}$, denoted by $\mathcal{K}(X,d_c)\coloneqq(V^\mathcal{K},E^\mathcal{K},w)$, with node set $V^\mathcal{K}\coloneqq X\cup\{0_\Omega\}$ and edge weights $w(u,v)\coloneqq d_c(u,v)$ for all $(u,v)\in E^\mathcal{K}$. Note that any edge entering the root can be removed.
    \item Solve the \textsc{Minimum Spanning Arborescence Problem} (\MSA) on $\mathcal{K}(X,d_c)$  with $0_\Omega$ as root node.
\end{enumerate}
For a connected weighted directed graph $G$ and a root node $r$ in $G$, the \msalong (\MSA) asks for a spanning arborescence $A^\star$ on $G$, which is rooted in $r$ and has minimum weight $w(A^\star)$, where $w(A)\coloneqq\sum_{(u,v)\in A}w(u,v)$ \cite{edmonds:1967aa}. 
\MSA can be solved in polynomial time, \eg, in $O(|V^G|^2)$ time with Tarjan's implementation \cite{tarjan:1977aa} of Edmonds' algorithm \cite{edmonds:1967aa}. Hence, the main difficulty of the problem consists in computing the edge weights in $\mathcal{K}$, \ie, the edit distances between elements of $X$.

\begin{lemma}\label{lem:inP}
As long as the edit space $(\Omega,d_c)$ allows for a polynomial time computation of minimum-cost edit paths, the corresponding version of \MEA belongs to the complexity class $\mathcal{P}$.
\end{lemma}
\begin{proof}
    By assumption, 
    $d_c$ is computed in polynomial time by some algorithm \ALGDIST that is called $O(n^2)$ times with complexity $O_{\ALGDIST}$ in order to generate the complete graph 
    $\mathcal{K}(X,d_c)$. So, \MEA can be solved in 
    $O(n^2O_{\ALGDIST}+(n+1)^2)$ time complexity by using Tarjan's implementation of Edmond's algorithm.\qed
\end{proof}

\paragraph{\textbf{A Heuristic for Non-Polynomial Cases.}}
\begin{algorithm}[!t]
\caption{A generic heuristic for $\MEA$.}\label{alg:generic_heuristic}
\begin{algorithmic}[1]
\REQUIRE A finite set $X$ of elements from an edit space $(\Omega,d)$ with origin $0_\Omega$, a parameter $k\in[0,|X-1|]$, two edit distance heuristics \ALGONE and \ALGTWO.
\ENSURE A low-cost edit arborescence $(\mathcal{A},\Psi)$ for the $\MEA$ problem.
\STATE Set $x_0\coloneqq 0_\Omega$ and initialize auxiliary graph $\mathcal{K}(X\cup x_0,E^{\mathcal{K}},w)$ with $E^{\mathcal{K}}\coloneqq\{x_0\}\times X$.\label{algo:compression:edges-start}
\STATE\textbf{for} $x\in X$ \textbf{do} Sample $k$ children $\Tilde{X}\in\binom{X\setminus\{x\}}{k}$ and set $E^\mathcal{K}\coloneqq E^\mathcal{K}\cup(\{x\}\times \Tilde{X})$.
\IF{prior information available}
Add promising edges to $E^{\mathcal{K}}$.\label{algo:compression:edges-end}
\ENDIF
\FOR{$ (x_i,x_j) \in E^\mathcal{K}$}\label{algo:compression:weights-start}
 \IF{$i=0$}
    Analytically compute the edit path $P_{ij}$.
 \ELSIF{identifiers available}
 Compute edit path $P_{ij}$ induced by identifiers.\label{algo:compression:identifiers}
 \ELSE{} Call \ALGONE to compute low-cost edit path $P_{ij}$.
 \ENDIF
\STATE Set $w(x_i,x_j)\coloneqq c(P_{ij})$.\label{algo:compression:weights-end}
\ENDFOR
\STATE Run Edmonds' algorithm on $\mathcal{K}$ to obtain $\bigA$\label{algo:compression:edmonds}.
\STATE\textbf{if} tightening \textbf{then for} $i\in [1,n]$ \textbf{do} Call \ALGTWO to compute tighter edit path $P_{\bigA[i]i}$.\label{algo:compression:tightening}
 \FOR{$ i \in [1,n]$}\label{algo:compression:return-start}
    Set $\Psi[i]\coloneqq P_{\bigA[i]i}$.
\ENDFOR
\STATE \textbf{return} $(\mathcal{A},\Psi)$\label{algo:compression:return-end}
\end{algorithmic}
\end{algorithm}
We adapt the algorithm described above to cases where the edit distance is not solvable in polynomial time. Given a set $X\subset\Omega$, \Cref{alg:generic_heuristic} computes a \emph{low-cost} edit arborescence $(\bigA,\Psi)\in\spaceA(X)$ based on approximations or heuristics to estimate the edit distance. It starts by constructing a size-reduced auxiliary digraph $\mathcal{K}$ (lines \ref{algo:compression:edges-start} to \ref{algo:compression:edges-end}) that connects $0_\Omega$ to each element $x_i\in X$, and each $x_i$ to $k\leq|X|-1$ randomly selected elements of $X\,{\setminus}\,\{x_i\}$. If some promising edges are known \emph{a priori} (\eg, if $X$ has an implicit internal structure), they are added to $\mathcal{K}$.
Then, \Cref{alg:generic_heuristic} computes optimal or low-cost edit paths whose costs provide weights for the edges of $\mathcal{K}$ (lines \ref{algo:compression:weights-start} to \ref{algo:compression:weights-end}). For $0_\Omega$' out-edges, optimal edit paths can be computed analytically (insertions only). If identifying attributes are available for all elements of $X$ (\eg, unique node labels if $X$ is a set of graphs), it is sometimes possible to compute optimal edit paths from these identifiers. Otherwise, low-cost edit paths are computed by calling a polynomial edit distance heuristic \ALGONE. Once all edge weights for $\mathcal{K}$ have been computed, an optimal arborescence $\bigA$ on $\mathcal{K}$ is constructed by Edmonds' algorithm (line \ref{algo:compression:edmonds}). Optionally, a tighter edit distance heuristic \ALGTWO can be called to shorten the paths in $\bigA$ before returning the edit arborescence (line \ref{algo:compression:tightening}).

\section{\mgealong}\label{sec:mgea}

In the remainder of the paper, we will focus on the specific case of \MEA where the space $\Omega$ is a space of labeled graphs. 

\paragraph{\textbf{Graphs.}}
We assume 
that graphs are finite, simple, undirected, and labeled. However, all presented techniques can be straightforwardly adapted to directed or unlabeled graphs.
A labeled graph $G$ is a four-tuple $G\coloneqq(V^G,E^G,\ell^G_V,\ell^G_E)$, where $V^G$ and $E^G$ are sets of nodes and edges, while  $\ell^G_V:V^G\to\Sigma_V$ and $\ell^G_E:E^G\to\Sigma_E$ are labeling functions that annotate nodes and edges with labels from alphabets $\Sigma_V$ and $\Sigma_E$, respectively. $\spaceG(\Sigma_V,\Sigma_E)$, or $\spaceG$ for short, denotes the set of all graphs for fixed alphabets $\Sigma_V$ and $\Sigma_E$. $0_\spaceG$ denotes the empty graph (the null element of $\spaceG$). Two graphs $G,H\in\spaceG$ are \emph{isomorphic}, denoted by $G\simeq H$, if and only if there is a bijection between 
$V^G$ and $V^H$ that preserves both edges and labels.

\paragraph{\textbf{Edit Operations and Edit Paths.}}\label{sec:operations}
  We consider the following elementary edit operations, where $\epsilon$ is a dummy node and $\epsilon_\ell$ is a dummy label:
\begin{itemize}
 \item Node deletion (nd): $(v,\epsilon_\ell)$, with $v\in V^G$ isolated.
 \item Edge deletion (ed): $(e,\epsilon_\ell)$, with $e\in E^G$.
 \item Node relabeling (nr): $(v,\ell)\in V^G\times(\Sigma_V\setminus\{\ell^G_V(v)\})$.
 \item Edge relabeling (er): $(e,\ell)\in E^G\times(\Sigma_E\setminus\{\ell^G_E(e)\})$.
 \item Node insertion (ni): $(\epsilon,\ell)$, with $\ell\in\Sigma_V$.
 \item Edge insertion (ei): $(e,\ell)\in(\binom{V^G}{2}\setminus E^G)\times\Sigma_E$.
\end{itemize}
For each edit path $P$ composed of such 
operations, there are many equivalent edit paths with a same edit cost, obtained just by reordering the operations in $P$. In particular, as the deletion of a node assumes that its incident edges have been previously deleted, these operations can be replaced by node-edge deletions (ned): delete all the edges incident to a node and then delete this node. So we can distinguish two different types of edge deletions:
implied edge deletion (i-ed), \ie, an edge deletion in a node-edge deletion, and non-implied edge deletion (ni-ed), \ie, an edge deletion between two nodes that are not deleted by $P$.
The cost of an edit path $P$ can thus be rewritten as
$
    c(P)=\sum_{t\in T}\sum_{o\in P^t}c^t(o)\text{,}
$
where $P^t$ is the (possibly empty) set of all edit operations of type $t\in T$, with $T\coloneqq\{\text{ni-ed},\text{i-ed},\text{nd},\text{nr},\text{er},\text{ni},\text{ei}\}$, and $c^t$ is an edit cost function for type $t$.

\begin{remark}\label{rem:sequence}
Any concatenation $\sigma_{\text{ni-ed}}(P^{\text{ni-ed}})\sqcup\sigma_{\text{i-ed}}(P^{\text{i-ed}})\sqcup\ldots\sqcup\sigma_{\text{ei}}(P^{\text{ei}})$ of edit operations, with $\sigma_t$ a permutation on $P^t$, defines an edit path equivalent to $P$.
\end{remark}
\begin{remark}\label{rem:constant_costs}
$c(P)=\sum_{t\in T}c^t|P^t|$ if $c_t$ is a constant for each type of operation $t$.
\end{remark}
\paragraph{\textbf{Node Maps and Induced Edit Paths.}}
A \emph{node map} (or \textit{error-correcting bipartite matching}) between a graph $G$ and a graph $H$ is a relation $\pi\in(V^G\cup\{\epsilon\})\times(V^H\cup\{\epsilon\})$ such that the following two conditions hold:
\begin{itemize}
\item For each node $u\in V^G$, there is exactly one node $v\in V^H\cup\{\epsilon\}$ such that $(u,v)\in\pi$. We denote this node $v$ by $\pi(u)$. \item For each node $v\in V^H$, there is exactly one node $u\in V^G\cup\{\epsilon\}$ such that $(u,v)\in\pi$. We denote this node $u$ by $\pi^{-1}(v)$.
\end{itemize}
Let $\Pi(V^G,V^H)$ be the set of all node maps. Each node map $\pi\in\Pi(G,H)$ can be transformed into an edit path, denoted by $P[\pi]$ (\textit{induced edit path}), such that, for each $(u,v)\in\pi$, there is a corresponding edit operation $(u,\ell)$: $u$ is deleted if $v=\epsilon$, it is relabeled if $(u,v)\in V^G\times V^H$ and $\ell_V^G(u)\not=\ell_V^H(v)$, or a new node is inserted if $u=\epsilon$. Operations on edges are induced by the operations on nodes, \ie, from the pairs $((u,\pi(u)),(v,\pi(v))$ with $u,v\in V^G\cup\{\epsilon\}$. For details, we refer to \cite{blumenthal:2020aa}. What is important here is that any type of edit operation is taken into account by a node map. In particular, implied and non-implied edge deletions can be distinguished with a specific cost for each type.
\paragraph{\textbf{Graph Edit Distance.}}
The cost of an optimal edit path from a graph $G$ to a graph $H^\prime\simeq H$ defines the graph edit distance (GED) from $G$ to $H$ ($d_c$ with graph isomorphism as equivalence relation):
$\GED(G,H)\coloneqq{\min}_{P \in \mathcal{P}(G,H)}\,c(P)$. 
\GED is hard to compute and approximate, even when restricting to simple special cases \cite{zeng:2009aa,blumenthal:2019ad}. However, many heuristics are able to reach tight upper and/or lower bounds. 
They are based on a reformulation of \GED as 
an \textsc{Error-Correcting Graph Matching Problem}: 
$\GED(G,H)={\min}_{\pi\in\Pi(V^G,V^H)}\,c(P[\pi])$, which is equivalent to the above definition under mild assumptions on the edit cost function $c$.
We refer to \cite{riesen:book:2016,blumenthal:2020aa} for an overview.
\paragraph{\textbf{Problem Formulation and Hardness.}}
We can now define \MGEA:
\begin{definition}[\MGEA]
The \mgealong (\MGEA) is a \MEA problem with $\Omega\,{\coloneqq}\,\spaceG$, $X\,{\coloneqq}\,\{G_1,...,G_n\}$ and $d_c\coloneqq\GED$.
\end{definition}
As the problem of computing \GED is $\NP$-hard, \Cref{lem:inP} does not apply here. 
\begin{theorem}\label{thm:hardness}
 \MGEA{} is \NP-hard.
\end{theorem}
\begin{proof}
The Hamiltonian cycle problem is an \NP-hard problem which asks to decide if an undirected, unlabeled $G$ contains a cycle passing through all of its nodes. Let $X\coloneqq\{G,H\}$ be a set of graphs with $H$ a cycle on $|V^G|$ nodes, and let $(\bigA^\star,\Psi^\star)$ be an optimal edit arborescence for X. Moreover, we assume \mywlog that $G$ contains more edges than $H$. 
There are three different edit arborescences for $X$, \ie, $\bigA(X)=\{(\bigA_1,\Psi_1)(\bigA_2,\Psi_2),(\bigA_3,\Psi_3)\}$ with edges $E^{\bigA_1}=\{(0_\spaceG,H),(H,G)\}$, $E^{\bigA_2}=\{(0_\spaceG,H),(0_\spaceG,G)\}$, and $E^{\bigA_3}=\{(0_\spaceG,G),(G,H)\}$, respectively. 

Assume that $G$ is Hamiltonian. Then, by construction of $H$, an optimal node map $\pi^\star_{H,G}\in\argmin_{\pi\in\Pi(G,H)}c(P[\pi])$ induces an edit path that contains only edge insertions, the edges of $G$ not mapped to an edge of $H$. So $c(\Psi_1)=c(P[\pi_{0_\spaceG,G}])$. As $H$ is non-empty and $G$ has more edges than $H$, it holds that $c(\Psi_1)<c(\Psi_2)$. Similarly, if edge deletion costs are strictly positive, then $c(\Psi_1)<c(\Psi_3)$, otherwise $c(\Psi_1)=c(\Psi_3)=c(P[\pi_{0_\spaceG,G}])$. This implies $c(\Psi^\star)=c(\Psi_1)=c(P[\pi_{0_\spaceG,G}])$ in all cases.

For the reciprocal, we show its contraposition. Assume that $G$ is not Hamiltonian. By choice of $G$ and $H$, an optimal node map $\pi^\star_{H,G}\in\argmin_{\pi\in\Pi(G,H)}c(P[\pi])$ induces an edit path composed of both edge insertions from $G$ and edge removals from $H$. So the concatenation of the path induced by $\pi_{0_\spaceG,H}$ and the path induced by $\pi^\star_{H,G}$ contains $|E^G|$ edge insertions and at least one insertion and deletion of a same edge in $H$. By consequence, $c(\Psi_1)>c(P[\pi_{0_\spaceG,G}])$. Similarly, we obtain $c(\Psi_2)>c(P[\pi_{0_\spaceG,G}])$ and $c(\Psi_3)>c(P[\pi_{0_\spaceG,G}])$. This implies $c(\Psi^\star)\not=c(P[\pi_{0_\spaceG,G}])$.

Let $P$ be the path in $\bigA^\star$ from the root $0_\spaceG$ to $G$, and $\pi^\star_{H,G}$ be an optimal node map from $H$ to $G$:
\begin{eqnarray}
c(\bigA^\star)&\geq c(P)\geq c(P[\pi^\star_{G_0,G}])\label{eq:hardness}
\end{eqnarray}
Moreover, the following statements hold:
\begin{itemize}
\item An edit path from the empty graph $G_0$ to any other graph is optimal if and only if it contains only insertion operations.
\end{itemize}
Since $G$ contains more edges than $H$ and $G$ is non-empty, we know that $c(\bigA)>|\enc(G)|-\beta_G$ holds for all edit arborescence with topology $E^{\bigA}_2$ or $E^{\bigA}_3$. Moreover, the above statements imply that the edit arborescence $(\{G_0,H,G\},E^{\bigA}_1,\{\pi^\star_{G_0,H},\pi^\star_{H,G}\})$ has cost $|\enc(G)|-\beta_G$ if and only if $G$ is Hamiltonian. By \cref{eq:hardness}, this proves the theorem.
\end{proof}
 -------------------------------------------

\section{Arborescence-Based Compression}\label{sec:abc}

In this section, we show how to leverage \MGEA for compressing a set of labeled graphs. For this, we introduce reconstructible and non-reconstructible edit arborescences, formulate the \abclong (\ABC), and present an encoding for induced edit paths.

\paragraph{\textbf{Reconstructible Edit Arborescence.}}
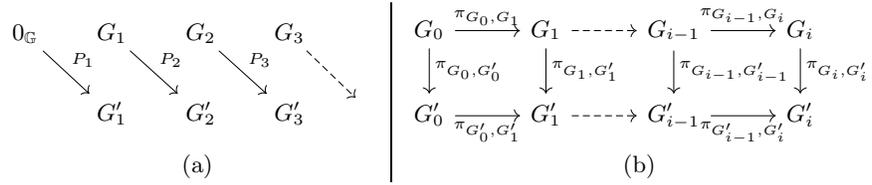
\begin{figure}[!t]
\centering
\begin{tabular}{c@{~}|@{~\,}c}
\begin{tikzcd}[cramped,column sep=1.8em]
0_\mathbb{G}\arrow[rd,"P_1"] & G_1\arrow[rd,"P_2"] &  G_{2}\arrow[rd,"P_3"] & G_3\arrow[rd,dashrightarrow] & \\
&G^\prime_1 & G^\prime_2 & G^\prime_{3}&{\color{white}G}
\end{tikzcd}
&
\begin{tikzcd}[column sep=2.7em]
    G_0\arrow[r,"\pi_{G_0,G_1}"]\arrow[d,"\pi_{G_0,G^\prime_0}"] & 
    G_1\arrow[d,"\pi_{G_1,G^\prime_1}"]\arrow[r,dashrightarrow] & G_{i-1}\arrow[d,"\pi_{G_{i-1},G^\prime_{i-1}}"]\arrow[r, "\pi_{G_{i-1},G_i}"] & G_i\arrow[d,"\pi_{G_i,G^\prime_i}"]\\
    G^\prime_0\arrow[r,"\pi_{G^\prime_0,G^\prime_1}"'] & 
    G^\prime_1 \arrow[r,dashrightarrow] & G^\prime_{i-1}\arrow[r,"\pi_{G^\prime_{i-1},G^\prime_i}"'] & G^\prime_i
\end{tikzcd}\\
(a)&(b)
\end{tabular}
\caption{(a) Paths in a non-reconstructible edit arborescence. (b) Composition of node maps used to construct reconstructible edit arborescences.}\label{fig:non-recons}
\end{figure}
When $d_c\coloneqq\GED$, since the definition of \GED is based on graph isomorphism,
applying an induced edit path $P[\pi_{G,H}]$ to a graph $G$ yields a graph $H^\prime\simeq H$. 
When using such edit paths within an edit arborescence $(\bigA,\Psi)\in\spaceA(X)$, with $X\subset\spaceG$ a set of graphs, the configuration described in \Cref{fig:non-recons}(a) occurs. Namely, the edit paths in $\Psi$ may be disjoint due to the isomorphism relation between the source graphs $G_i$ and target graphs $G_i^\prime$. Thus, the arborescence does not allow the reconstruction of any graph that is not directly connected to the root element. In this sense, only specific edit arborescences 
allow to reconstruct all graphs in $X$ up to isomorphism.
\begin{definition}[Reconstructability]\label{def:enc:ea}
An edit arborescence $(\bigA,\Psi)\in\spaceA(X)$ is \emph{reconstructible} if and only if each graph $G_i\in X$ can be constructed up to isomorphism by applying the sequence of edit paths $P\coloneqq(P_{s^1}, P_{s^2},\ldots, P_i)$ to the empty graph $0_\mathbb{G}$, where $(0_\spaceG,G_{s^1},G_{s^2},\ldots,G_i)$ is the path from $0_\spaceG$ to $G_i$ in $\bigA$.
\end{definition}
By composition of node maps (\Cref{fig:non-recons}(b)), it is easy to show 
the following property (proof omitted due to space constraints).
\begin{lemma}\label{lm:reconstruction}
For any edit arborescence $(\bigA,\Psi)$ on a set of graphs $X$, there is a reconstructible edit arborescence  $(\bigA,\Psi^\prime)$ on a set $X^\prime$, such that $X^\prime$ is isomorphic to $X$ and it holds that $c(\Psi^\prime)=c(\Psi)$.
\end{lemma}
By \Cref{def:enc:ea}, the set of graphs $X$ can be reconstructed up to graph isomorphism based on the encoding of a reconstructible edit arborescence.

\paragraph{\textbf{Problem Formulation.}}
For compressing a finite set of graphs $X\subset\spaceG$, we are interested in finding a reconstructible edit arborescence $(\bigA,\Psi)\in\spaceA(X)$ with a small encoding size $|\enc(\bigA,\Psi)|$, where $\enc(\cdot)$ denotes the encoding (a binary string) for the code $\enc$. Ideally we would like to minimize this size over $\spaceA(X)$ and all possible codes. To encode an edit arborescence, we encode both of its elements, \ie, the arborescence $\bigA$ defined as a sequence of indices, and the sequence $\Psi$ of edit paths induced by the node maps. In order to derive a useful expression for the code length function $|\enc(\cdot)|$, the edit path encodings  are concatenated. Thus, the encoding size to optimize is given by
$|\enc(\bigA,\Psi)|=|\enc(M_\bigA)|+|\enc(\bigA)|+{\sum}_{P \in\Psi} |\enc(P)|$,
where $M_\bigA$ is the overhead for decoding the different parts of $\enc(\bigA,\Psi)$. In order to optimize this size
, we must define an \emph{encoding size preserving} cost function which forces the encoding sizes of edit paths to coincide with their edit cost.
\begin{definition}[Encoding Size Preservation]\label{def:esp}
Let $\enc$ be a code for edit paths. An edit cost function $c$ is \emph{encoding size preserving} \wrt code $\enc$ if and only if there is a constant $\gamma$ such that $|\enc(P)|=c(P)+\gamma$ holds for any edit path $P$. Put differently, an encoding size preserving cost function assigns to each edit operation the space required in memory to encode the operations with code $\enc$. 
\end{definition}
Assuming that a code $\enc$ and an encoding size preserving cost function $c$ \wrt $\enc$ exist, the encoding size for any edit arborescence $(\bigA,\Psi)\in\spaceA(X)$ can be rewritten as $|\enc(\bigA,\Psi)|=|\enc(M_\bigA)|+|\enc(\bigA)|+c(\Psi)+\gamma|X|$. Since the encoding size for 
$\bigA$ depends only on the number of nodes, the problem of minimizing $|\enc(\bigA,\Psi)|$ amounts to minimizing $c(\Psi)$. Consequently, finding a compact encoding of a set of graphs $X$ reduces to a \MGEA problem as introduced in \Cref{sec:mgea}.
\begin{definition}[\ABC]
Let $X\coloneqq\{G_1,\ldots,G_n\}\subset\spaceG$ be a finite set of graphs, $\enc$ be a code for edit paths, and $c$ be an encoding size preserving edit cost function for $\enc$. Then, the \abclong (\ABC) asks for a minimum weight reconstructible edit arborescence $(\mathcal{A},\Psi^\prime)$ on some set of graphs $X^\prime\coloneqq\{G_1^\prime,\ldots,G_n^\prime\}$ such that, for all $i\in[1,n]$, $G^\prime_i\simeq G_i$.
\end{definition}
We stress that, thanks to the use of encoding size preserving edit costs, the value that is optimized by \ABC corresponds to the length of the code $\enc(\bigA,\Psi^\prime)$ up to a constant. In other words, solving $\ABC$ produces the most compact arborescence-based representation of $X$. Given the simple correspondence between reconstructible edit arborescences and their non-reconstructible counterparts, 
\ABC{} reduces to \MGEA{} by restricting to encoding size preserving edit costs. Since \MGEA is \NP-hard (Theorem~\ref{thm:hardness}), we propose to heuristically compute a low-cost edit arborescence as detailed below.

\paragraph{\textbf{Heuristic Solver for ABC.}}
\begin{algorithm}[!t]
\caption{\ABC encoding of graph collections.}\label{algo:encoding}
\begin{algorithmic}[1]
\REQUIRE A set of graphs $X$, a code $\enc$, and an edit cost function $c$.
\ENSURE Encoding $\enc(\bigA,\Psi^\prime)$ of a reconstructible edit arborescence $(\bigA,\Psi^\prime)$ on $X$.
\STATE{Compute $(\mathcal{A},\Psi)$ with Algorithm \ref{alg:generic_heuristic}.\label{algabc:init}}
\STATE{Initialize list $L\coloneqq[(0_\spaceG,\pi_{\text{id}})]$, where $\pi_{\text{id}}$ is the identity.}
\STATE{Initialize encoding $\enc(\bigA,\Psi^\prime)\coloneqq\enc(M_\bigA)\enc(E^\bigA)$.}
\WHILE{$L\not=\emptyset$\label{algabc:initBFS}}
\STATE{Pop an element $(G_i,\pi_{G_i,G_i^\prime})$ from $L$.}
\FORALL{children $G_j$ of $G_i$ in $\bigA$} 
\STATE{Get $\pi\in\Pi(G_i,G_j)$ with $P[\pi]=\Psi[j]$ and initialize node ID $v^\prime\coloneqq1$.}\label{algabc:map-start}
\STATE{Initialize $\pi^\prime\in\Pi(G^\prime_i,G^\prime_j)$ as node map of insertions and deletions only.}
\STATE{\textbf{for all} $v\in V^{G_i}$ \textbf{if} $\pi(v)\neq\epsilon$ \textbf{then} Set $\pi^\prime(\pi_{G_i,G_i^\prime}(v))\coloneqq v^\prime$ and increment $v^\prime$.}\label{algabc:map-end}
\STATE{Concatenate $\enc(P[\pi^\prime])$ to $\enc(\bigA,\Psi^\prime)$.}\label{algabc:concat}
\IF{$G_j$ is no leaf in \bigA}
Append $(G_j,\pi^\prime\circ\pi_{G_i,G_i^\prime}\circ\pi^{-1})$ to $L$.\label{algabc:map-iso}
\ENDIF
\ENDFOR
\ENDWHILE
\State \textbf{return} $\enc(\bigA^\prime,\Psi^\prime)$
\end{algorithmic}
\end{algorithm}
Algorithm \ref{algo:encoding} sketches our strategy to tackle the \ABC{} problem.
Given a set $X$ of graphs, it first uses Algorithm \ref{alg:generic_heuristic}, which outputs a non-reconstructible edit arborescence $(\bigA,\Psi)$ on $X$. After initializing the code, it starts encoding a reconstructible edit arborescence by  going through the arborescence in BFS order (line~\ref{algabc:initBFS}). For each new node $G_j$ with parent node $G_i$, a node map $\pi^\prime$ from $G^\prime_i\simeq G_i$ to $G^\prime_j\simeq G_j$ is reconstructed (lines~\ref{algabc:map-start} to \ref{algabc:map-end}), and the code of its induced edit path is added to the code of the arborescence (line~\ref{algabc:concat}). If $G_j$ is not a leaf, the node map representing the isomorphism between $G_j$ and $G_j^\prime$ is computed for later use (line~\ref{algabc:map-iso}).
\begin{remark}[Star Ratio]\label{rem:star}
In the worst case, we obtain a star $\bigS\in\spaceA(X)$, which connects the empty graph $0_\spaceG$ to all the graphs in $X$. This yields the upper bound $|\enc(\bigA,\Psi)|\leq|\enc(\bigS)|$ on the encoding size of the obtained arborescence.  Since encoding a graph from the empty graph by insertion operations only is similar to encoding the graph itself, the encoding size for the star is close to the encoding size for $X$, for a similar encoding strategy. Consequently, the \emph{star ratio} $|\enc(\bigA)|/|\enc(\bigS)|$ provides a good indicator for the compression quality.
\end{remark}
\begin{remark}
The encoded structure is designed to allow a straightforward decompression of any graph $G_i$, which, starting from the empty graph, simply consists in consecutively applying the edit paths along the path from the root to $G_i$ in \bigA.
\end{remark}

\paragraph{\textbf{A Code for Induced Edit Paths.}}
We show that there is a code $\enc$ for edit paths and an edit cost function $c$ such that $c$ is encoding size preserving \wrt $\enc$. Using the notations introduced in \Cref{sec:operations}, 
we encode an edit path as the concatenated string $\enc(P[\pi])\coloneqq\enc(M_P)\enc(P^\textnormal{ni-ed})\enc(P^\textnormal{nd})\enc(P^\textnormal{nr})\enc(P^\textnormal{er})\enc(P^\textnormal{ni})\enc(P^\textnormal{ei})$, where $M_P$ denotes the overhead for decoding each string $\enc(P^t)$. Note that the set $P^\text{i-ed}$ is not encoded, since implied edge deletions can be implicitly represented by node deletions. Similarly, we encode a set of edit operations $P^t$ as $\enc(P^t)\coloneqq\enc(o^t_1)\enc(o^t_2)\cdots\enc(o^t_{|P^t|})$, with $o_i^t\in P^t$. Any edit operation $o\coloneqq(a,\ell)$ is encoded as $\enc(o)\coloneqq\enc(a)\enc(\ell)$, with $\enc(a)\coloneqq\emptyset$ if $a=\epsilon$, and $\enc(\ell)\coloneqq\emptyset$ if $\ell=\epsilon_\ell$. That is, the dummy elements $\epsilon$ and $\epsilon_\ell$ in deletion operations and node insertions are not encoded. Ultimately, the encoding size for $P[\pi]$ hence depends on how the nodes, edges, and their labels are encoded in the codes of the edit operations.

We consider fixed-length codes for nodes, edges, and their labels (other codes will be studied in future works). For a set $X\in\spaceG$, nodes are encoded as integers on $\beta_V$ bits, edges are encoded as a pairs of integers on $2\beta_V$ bits, and node or edge labels are encoded on, respectively, $\beta_{\Sigma_V}$ and $\beta_{\Sigma_E}$ bits. 
Dictionaries can be used for the labels and encoded in the overhead $M_\bigA$ or known \emph{a priori}. In order to decode each set of edit operation $P^t$, $M_P$ must contain their sizes $|P^t|$. They are encoded on $\beta_P$ bits 
for each edit path. We 
obtain
$|\enc(P[\pi])|=\beta_{P}+\sum_{t\in T}c^t|P^t|$, 
where 
$c^\textnormal{nr}\coloneqq\beta_V+\beta_{\ell_V}$, $c^\textnormal{nd}\coloneqq\beta_V$, $ c^\textnormal{ni}\coloneqq\beta_{\Sigma_V}$, $c^\textnormal{er}\coloneqq c^\textnormal{ei}\coloneqq 2\beta_V+\beta_{\Sigma_E}$, $c^\textnormal{ed-ni}\coloneqq 2\beta_V$, and $c^\textnormal{ed-i}\coloneqq 0$.
With these constant costs, the pair $(\enc,c)$ defined above is encoding size preserving (\Cref{rem:constant_costs} and \Cref{def:esp}) with constant $\beta_P$ for any node map $\pi$, \ie, $|\enc(P[\pi])|=c(P[\pi])+\beta_P$. 
Therefore, the encoding size for a spanning edit arborescence $(\bigA,\Psi)\in\spaceA(X)$ reduces to $|\enc(\bigA,\Psi)|=|\enc(M_\bigA)|+|\enc(\bigA)|+c(\Psi)+\beta_P|X|$, which implies that minimizing  $|\enc(\cdot)|$ is an \ABC problem.

\section{Experiments}\label{sec:xp}
We performed an empirical evaluation of our compression method in the context of data archiving. Since no dedicated algorithms for compressing graph collections exist in this context, we compared it to the generic tar.bz compression, sufficient to highlight the potential of our method. For the evaluation, we used compression ratio (size of compressed graph collection divided by size of original graph collection), compression and decompression times, and star ratio (\Cref{rem:star}). Further investigations, \eg, comparisons against graph compression methods, applied to each graph individually or to a disjoint union of the graphs in the collection, will be carried out in future work. Other generic compression tools such as zip or tar.gz yielded worse compression ratios than tar.bz in initial tests.

\paragraph{\textbf{Datasets.}}
We used eight different datasets (\Cref{tab:data}). The datasets \aids and \muta from the IAM Graph Database Repository \cite{riesen:2008aa}, and \acyclic, \pah, and \mao from GREYC's Chemistry Dataset\footnote{\url{https://brunl01.users.greyc.fr/CHEMISTRY/index.html}} contain graphs modeling chemical compounds. We also tested on time-evolving minimum spanning trees (MSTs) induced by the pairwise correlations of a large-scale U.\,S.\ stocks time series dataset.\footnote{\url{https://www.kaggle.com/borismarjanovic/price-volume-data-for-all-us-stocks-etfs}} 
Such MSTs are widely used for detecting critical market events such as financial crises \cite{coelho:2007ux,liu:2021aa}. 
We constructed three versions of the MSTs with the code in \cite{liu:2021aa}: \stocksfloat (edge labels are floating-point stocks correlations), \stocksint (the correlations are rounded to integers), and \stocksunl (no edge label). 
For all datasets, graphs were initially stored in GXL format.\footnote{\url{https://userpages.uni-koblenz.de/~ist/GXL/index.php}} 

\begin{table}[t]
\caption{Number of graphs $|X|$, maximum and average number of nodes $|V|$, as well as node and edge label alphabet sizes $|\Sigma_V|$ and $|\Sigma_E|$ for all datasets.}\label{tab:data}
\centering
\begin{tabular}{@{}lS[table-format=4.0]S[table-format=3.0]S[table-format=3.2]S[table-format=3.0]S[table-format=3.0]@{~}|@{~}lS[table-format=4.0]S[table-format=3.0]S[table-format=3.2]S[table-format=3.0]S[table-format=3.0]@{}}
\toprule
dataset & {$|X|$} & {max $|V|$} & {avg $|V|$} & {$|\Sigma_V|$} & {$|\Sigma_E|$} & dataset & {$|X|$} & {max $|V|$} & {avg $|V|$} & {$|\Sigma_V|$} & {$|\Sigma_E|$} \\
\midrule
\acyclic & 183 & 11 & 8.15 & 3 & 1 & \mao & 68 & 27 & 18.38 & 3 & 4\\
\muta & 4337 & 417 & 30.32 & 14 & 3 & \stocksunl & 1600 & 213 & 212.99 & 213 & 0\\
\aids & 1500 & 95 & 15.72 &  {$\infty$} & 3 & \stocksint & 1600 & 213 & 212.99 & 213 & 100\\
\pah & 94 & 28 & 20.7 & 1 & 1 & \stocksfloat& 1600 & 213 & 212.99 & 213 & {$\infty$} \\
\bottomrule
\end{tabular}
\end{table}

\paragraph{\textbf{Parameters and Implementation.}}
We tested two versions of our \ABC method (\Cref{algo:encoding})\,---\,
with and without additional tar.bz compression of the obtained codes. For both versions, the out-degree $k$ of all nodes 
in $\mathcal{K}$ was varied across $\{0.1\cdot|X|,0.2\cdot|X|,\ldots,1.0\cdot|X|\}$, and we did 5 repetitions for each value. For the experiments reported in Table \ref{tab:arbo}, we performed 10 repetitions for each dataset. For \stocks,
we also added all temporal edges to $\mathcal{K}$ and always used the node maps induced by the stock identities across time (cf.\ lines \ref{algo:compression:edges-end} and \ref{algo:compression:identifiers} in \Cref{alg:generic_heuristic}). On the other datasets, the node maps were computed and refined using the \GED heuristics  $\ALGONE\coloneqq\BRANCHUNIFORM$ \cite{zheng:2015aa} and $\ALGTWO\coloneqq\IPFP$ \cite{blumenthal:2020aa}.
All algorithms were implemented in C++ using the \GED library GEDLIB \cite{blumenthal:2019aa} and the \MSA library MSArbor \cite{fischetti93}.
\footnote{\url{https://github.com/lucasgneccoh/gedlib}} 
Tests were run on a Linux system with an Intel Haswell CPU (24 cores, 2.4\,GHz each) and 19\,GB of main memory.

\begin{figure}[!t]
\centering
\begin{tikzpicture}
\begin{groupplot}[
group style={group name=groupplot, group size=4 by 2, horizontal sep=0.9cm, vertical sep=1.4cm},
width=.3\linewidth,
height=.3\linewidth,
/tikz/font=\footnotesize]
\ratioplots{\acyclic}{acyclic.csv}{2}
\ratioplots{\muta}{Mutagenicity.csv}{2}
\ratioplots{\pah}{pah.csv}{2}
\ratioplots{\mao}{mao.csv}{2}
\ratioplots{\aids}{AIDS.csv}{2}
\ratioplots{\stocksfloat}{msts_float_w.csv}{2}
\ratioplots{\stocksint}{msts_int_w.csv}{2}
\ratioplots{\stocksunl}{msts_no_w.csv}{2}
\end{groupplot}
\node at ($(groupplot c1r1.north) !.5! (groupplot c4r1.north)$) [inner sep=0pt,anchor=south, yshift=4ex] {\pgfplotslegendfromname{ratiolegend}};
\end{tikzpicture}%
\caption{Mean compression ratios \wrt out-degree $k$ for tar.bz and \ABC w/ or w/o tar.bz. 
For \stocks, the values for $k=0$ in the plots correspond to the setting where the auxiliary graph $\mathcal{K}$ only contains temporal edges.}\label{fig:cratios}
\end{figure}

\paragraph{\textbf{Compression Ratio.}}
\Cref{fig:cratios} shows that, for all datasets except \pah, \ABC with tar.bz significantly outperformed tar.bz compression alone and led to smaller compression ratios than \ABC w/o tar.bz for all datasets except \acyclic. Using out-degrees $k>0.4\cdot|X|$ only marginally improved compression. For \stocks, using only temporal edges ($k=0$) led to very good results. Moreover, \stocksunl can be more compressed with \ABC than the other \stocks datasets, as cheaper edit paths can be computed for graphs with unlabeled edges. 

\paragraph{\textbf{Arborescence Structure, Star Ratio, and Runtime.}}
Columns 2 to 4 of \Cref{tab:arbo} provide statistics regarding the arborescences computed with $k=0.4\cdot|X|$. They seem to have a good balance between depth and width (
number of leaves vs. number of internal nodes). 
The star ratios (column 5) indicate how much space is gained by using \ABC \wrt encoding each graph separately with the same underlying encoding scheme (a star ratio of $1$ means no compression). Columns 6 to 9 summarize the \ABC compression and decompression times. 
The most important observation is that, although \ABC is much slower than tar.bz, the runtimes are still acceptable in application scenarios where a data holder wants to offer compressed graph datasets for download (compressing the largest dataset \muta took about four to five hours).
\begin{table}[!t]
\caption{Mean compression and decompression times (in sec.), and standard deviations, of \ABC with tar.bz for $k=0.4\cdot|X|$, as well as mean depths, star ratios, numbers of leafs $|\mathcal{L}|$ and inner nodes $|\mathcal{I}|$ of the computed arborescences.}\label{tab:arbo}
\centering
\small
\begin{tabular}{@{}lS[table-format=4.0]S[table-format=4.0]S[table-format=3.1]S[table-format=0.2]S[table-format=5.0]S[table-format=4.1]S[table-format=2.1]S[table-format=1.2]@{}}
\toprule
dataset & {$|\mathcal{L}|$}  & {$|\mathcal{I}|$} & {avg depth} & {star ratio} & \multicolumn{2}{c}{compression} & \multicolumn{2}{c}{decompression} \\
\cmidrule(lr){6-7}\cmidrule(l){8-9}
 &&&&& {mean} & {std} & {mean} & {std}  \\
\midrule
\acyclic & 65 & 118 & 19.5 & 0.37  & 8 & 0.6 & 0.3 & 0.02 \\
\muta & 1751 & 2586 & 95.4 & 0.47 & 16052 & 2385.0 & 14.6 & 1.34 \\
\aids & 641 & 859 & 46 & 0.63 & 673 & 59.0 & 5.5 & 0.31 \\
\pah & 35 & 59 & 13.8 &  0.38 & 16 & 1.3 & 0.3 & 0.04 \\
\mao & 23 & 45 & 13.3 & 0.17 &  13 & 1.6 & 0.2 & 0.02\\
\stocksunl & 133 & 1467 & 75.9 & 0.51 &  1662 & 31.9 & 15.0 & 0.61\\
\stocksint & 153 & 1446 & 441.5 & 0.71 & 2095 & 46.8 & 18.7 & 0.37 \\
\stocksfloat & 148 & 1452 & 426.3 & 0.71 & 2166 & 28.5 & 18.9 & 0.43\\
\bottomrule
\end{tabular}
\end{table}
Indeed, unlike compression, decompression is fast even on the largest datasets (a couple of seconds). 
Runtime variations \wrt $k$ are detailed in \Cref{fig:scalability} for four datasets. As expected, the time required for computing the arborescences increases linearly with $k$, and the runtime of the refinement phase is independent of $k$. As the refinement algorithm \IPFP is randomized, the runtimes of the refinement phase have a higher variability than the runtimes of the arborescence phase.

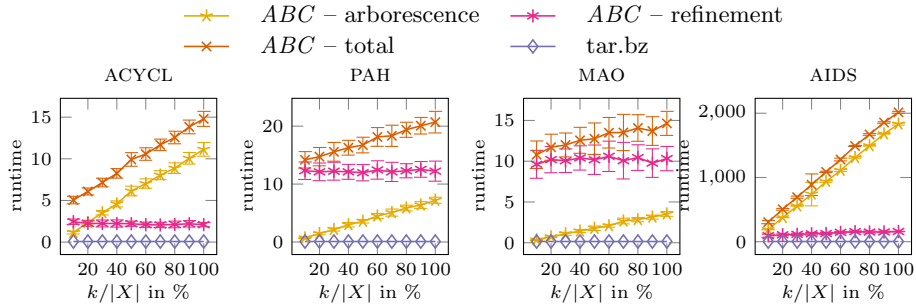
\begin{figure}[t]
\centering
\begin{tikzpicture}
\begin{groupplot}[
group style={group name=groupplot, group size=4 by 2, horizontal sep=1cm, vertical sep=1.5cm},
width=.3\linewidth,
height=.3\linewidth,
/tikz/font=\footnotesize]
\ctimeplots{\acyclic}{acyclic.csv}{2}
\ctimeplots{\pah}{pah.csv}{2}
\ctimeplots{\mao}{mao.csv}{2}
\ctimeplots{\aids}{AIDS.csv}{2}
\end{groupplot}
\node at ($(groupplot c1r1.north) !.5! (groupplot c4r1.north)$) [inner sep=0pt,anchor=south, yshift=3ex] {\pgfplotslegendfromname{ctimelegend}};
\end{tikzpicture}%
\caption{Means and standard deviations of runtimes (in sec.) for \ABC with tar.bz and its subroutines, and tar.bz alone. \ABC-total includes the final tar.bz step.}\label{fig:scalability}
\end{figure}



\section{Conclusions}\label{sec:conc}

In this paper, we have proposed the concept of an \emph{edit arborescence} and have introduced the \mealong (\MEA). \MEA yields a generic framework for inferring hierarchies in finite sets of complex data objects such as graphs or strings, which can be compared via \emph{edit distances}. We have shown how to leverage \MEA for the lossless compression of collections of labeled graphs\,---\,a task, for which no dedicated algorithms are available to date. Experiments on eight datasets show that our approach \ABC clearly outperforms standard compression tools in terms of compression ratio and that it achieves reasonable compression and decompression times. More precisely, the experiments showed that (1) on seven out of eight datasets, our \ABC method clearly outperformed tar.bz compression in terms of compression ratio; (2) compressing with \ABC is computationally expensive but still reasonable in settings where the compression is carried out by an institutional data holder; (3) decompression is much faster and only takes a couple of seconds even for the largest test datasets.


\bibliographystyle{splncs04}
\bibliography{sample}

\begin{thebibliography}{10}
\providecommand{\url}[1]{\texttt{#1}}
\providecommand{\urlprefix}{URL }
\providecommand{\doi}[1]{https://doi.org/#1}

\bibitem{adler01}
Adler, M., Mitzenmacher, M.: Towards compressing web graphs. In: Proceedings of
  the Data Compression Conference. p.~203. IEEE Computer Society (2001).
  \doi{10.5555/882454.875027}

\bibitem{Almodaresi20}
Almodaresi, F., Pandey, P., Ferdman, M., Johnson, R., Patro, R.: An efficient,
  scalable, and exact representation of high-dimensional color information
  enabled using de bruijn graph search. Journal of Computational Biology
  \textbf{27}(4),  485--499 (2020). \doi{10.1089/cmb.2019.0322}

\bibitem{besta18}
Besta, M., Hoefler, T.: Survey and taxonomy of lossless graph compression and
  space-efficient graph representations. CoRR  \textbf{arXive:1806.01799
  [cs.DS]} (2018)

\bibitem{besta19b}
Besta, M., Peter, E., Gerstenberger, R., Fischer, M., Podstawski, M., Barthels,
  C., Alonso, G., Hoefler, T.: Demystifying graph databases: Analysis and
  taxonomy of data organization, system designs, and graph queries. CoRR
  \textbf{arXive:1910.09017 [cs.DB]} (2019)

\bibitem{blumenthal:2019ad}
Blumenthal, D.B.: New Techniques for Graph Edit Distance Computation. Ph.D.
  thesis, Free University of Bozen-Bolzano (2019)

\bibitem{blumenthal:2020aa}
Blumenthal, D.B., Boria, N., Gamper, J., Bougleux, S., Brun, L.: Comparing
  heuristics for graph edit distance computation. {VLDB} J.  \textbf{29}(1),
  419--458 (2020). \doi{10.1007/s00778-019-00544-1}

\bibitem{blumenthal:2019aa}
Blumenthal, D.B., Bougleux, S., Gamper, J., Brun, L.: {{GEDLIB}}: A {{C++}}
  library for graph edit distance computation. In: GbRPR 2019. LNCS, vol.
  11510, pp. 14--24. Springer, Cham (2019). \doi{10.1007/978-3-030-20081-7\_2}

\bibitem{BOOKSTEIN91}
Bookstein, A., Klein, S.: Compression of correlated bit-vectors. Information
  Systems  \textbf{16}(4),  387--400 (1991). \doi{10.1016/0306-4379(91)90030-D}

\bibitem{chwatala09}
Chwatala, A.M., Raidl, G.R., Oberlechner, K.: Phylogenetic comparative methods.
  J. Math. Model. Algorithms  \textbf{8},  293–334 (2009).
  \doi{10.1007/s10852-009-9109-1}

\bibitem{coelho:2007ux}
Coelho, R., Gilmore, C.G., Lucey, B., Richmond, P., Hutzler, S.: The evolution
  of interdependence in world equity markets—evidence from minimum spanning
  trees. Physica A  \textbf{376},  455--466 (2007).
  \doi{https://doi.org/10.1016/j.physa.2006.10.045}

\bibitem{CORNWELL2017R333}
Cornwell, W., Nakagawa, S.: Phylogenetic comparative methods. Curr. Biol.
  \textbf{27}(9),  R333--R336 (2017). \doi{0.1016/j.cub.2017.03.049}

\bibitem{edmonds:1967aa}
Edmonds, J.: Optimum branchings. J. Res. Natl. Bur. Stand. B  \textbf{71}(4),
  233--240 (1967). \doi{10.6028/jres.071b.032}

\bibitem{fischetti93}
Fischetti, M., Toth, P.: An efficient algorithm for the min-sum arborescence
  problem on complete digraphs. {INFORMS} J. Comput.  \textbf{5}(4),  426--434
  (1993). \doi{10.1287/ijoc.5.4.426}

\bibitem{guralnik:1999wp}
Guralnik, V., Srivastava, J.: Event detection from time series data. In:
  Fayyad, U.M., Chaudhuri, S., Madigan, D. (eds.) {SIGKDD} 1999. pp. 33--42.
  {ACM} (1999). \doi{10.1145/312129.312190}

\bibitem{liu:2021aa}
Liu, T., Coletti, P., Dignös, A., Gamper, J., Murgia, M.: Correlation graph
  analytics for stock time series data. In: EDBT 2021 (2021),
  \url{https://edbt2021proceedings.github.io/docs/p173.pdf}

\bibitem{luo:2017uk}
Luo, J., Narasimhan, K., Barzilay, R.: Unsupervised learning of morphological
  forests. Trans. Assoc. Comput. Linguistics  \textbf{5},  353--364 (2017)

\bibitem{moschitti-etal-2006-semantic}
Moschitti, A., Pighin, D., Basili, R.: Semantic role labeling via tree kernel
  joint inference. In: CoNLL 2006. pp. 61--68. {ACL} (2006)

\bibitem{riesen:book:2016}
Riesen, K.: Structural Pattern Recognition with Graph Edit Distance:
  Approximation, Algorithms and Applications. Advances in Computer Vision and
  Pattern Recognition, Springer (2016). \doi{10.1007/978-3-319-27252-8}

\bibitem{riesen:2008aa}
Riesen, K., Bunke, H.: {IAM} graph database repository for graph based pattern
  recognition and machine learning. In: S+SSPR 2008. LNCS, vol.~5342, pp.
  287--297. Springer, Berlin, Heidelberg (2008).
  \doi{10.1007/978-3-540-89689-0\_33}

\bibitem{sourek2021lossless}
Sourek, G., Zelezny, F., Kuzelka, O.: Lossless compression of structured
  convolutional models via lifting. CoRR  \textbf{arXiv:2007.06567 [cs.LG]}
  (2021)

\bibitem{tarjan:1977aa}
Tarjan, R.E.: Finding optimum branchings. Networks  \textbf{7}(1),  25--35
  (1977). \doi{10.1002/net.3230070103}

\bibitem{zeng:2009aa}
Zeng, Z., Tung, A.K.H., Wang, J., Feng, J., Zhou, L.: Comparing stars: On
  approximating graph edit distance. Proc. VLDB Endow.  \textbf{2}(1),  25--36
  (2009). \doi{10.14778/1687627.1687631}

\bibitem{zheng:2015aa}
Zheng, W., Zou, L., Lian, X., Wang, D., Zhao, D.: Efficient graph similarity
  search over large graph databases. IEEE Trans. Knowl. Data Eng.
  \textbf{27}(4),  964--978 (2015). \doi{10.1109/TKDE.2014.2349924}

\end{thebibliography}
\end{document}